\documentclass[submission,copyright,creativecommons]{eptcs}
\providecommand{\event}{ICLP 2026} 

\usepackage{iftex}
\usepackage{amsmath}
\usepackage{amssymb}
\usepackage{amsthm}
\usepackage[inline]{enumitem}
\usepackage{tikz-cd}
\usepackage{cancel}

\newtheorem{theorem}{Theorem}

\theoremstyle{definition}
\newtheorem{definition}{Definition}

\newtheorem{example}{Example}

\usepackage{graphicx}

\ifpdf
  \usepackage{underscore}         
  \usepackage[T1]{fontenc}        
\else
  \usepackage{breakurl}           
\fi

\title{Logic Programming Semantics for Causal Processes}
\author{
Felix Weitkämper
\institute{German University Of Digital Science\\
Potsdam, Germany}
\email{felix.weitkaemper@german-uds.de}
}
\def\titlerunning{Logic Programs and Causal Processes }
\def\authorrunning{Felix Weitkämper}
\begin{document}
\maketitle

\begin{abstract}
Motivated by challenging modelling issues in the life sciences, we investigate the relationship between logic programming semantics and the eventual states of causal processes compatible with those logic programs. 
More precisely, we show that while stable models of positive logic programs correspond to the eventual states of processes commencing from a neutral state and continuing undisturbed indefinitely, supported models describe the eventual states reachable from arbitrary starting points. 
This also contributes to the discussion of the appropriate semantics for logic programming as a causal rule language, adding a temporal perspective to recent interpretations of the stable and supported model semantics from an explanatory viewpoint of causality. 
\end{abstract}

\section{Introduction}
Logic programming is among the most well-developed formalisms for modelling in artificial intelligence, with mature and efficient interpreters for Prolog, Datalog and answer set programming and close to half a century of rigorous theoretical work on its semantics \cite{DeclarativeLogicProgramming}. 

It can also be viewed as an incarnation of a calculus of causal rules, whose semantic underpinnings as a theory of causality have been traced back to fundamental philosophical principles \cite{bochman,RueckschlossWeitkaemper:2025,EelinkRW26}. 
Following the theory developed by Bochman \cite{bochman} and Eelink et al.\ \cite{EelinkRW26},  R\"uckschlo{\ss} and Weitk\"amper \cite{RueckschlossWeitkaemper:2025}  show that the supported model semantics emerges naturally from considering only those structures for which every true proposition can be justified by an explanation along the causal rules of the logic program. 
However, in the presence of cycles among the rules, such explanations can be cyclical themselves; therefore,  R\"uckschlo{\ss} and Weitk\"amper \cite{RueckschlossWeitkaemper:2025} argue that the stable model semantics is more meaningful, since it considers precisely those models where causal explanations can be ultimately grounded outside the causal system encoded in the set of rules. 

Following Pearl \cite{Causality}, this line of research describes causal knowledge without any reference to temporal order. 
However, in many cases, it is more natural to model influences as directed in time. 
Under this lens, causal cycles are simply concise descriptions of possible feedback loops. 

\begin{example}\label{ex:burns}
Consider the following example from Eelink et al.\ \cite{EelinkRW26},  R\"uckschlo{\ss} and Weitk\"amper \cite{RueckschlossWeitkaemper:2025}: 
Let $h_1$ and $h_2$ be two neighbouring houses, and consider the logic program with the two clauses 
\begin{align*}
\mathrm{burns}(h_1) \leftarrow \mathrm{burns}(h_2). \\ 
\mathrm{burns}(h_2) \leftarrow \mathrm{burns}(h_1).
\end{align*}
Then we can view this either as a set of explanatory rules, where the burning of one house explains the burning of the other, or we can view this as a blueprint for a process that leads from the burning of one house at a given time $t$ to the burning of the other house at a time $s > t$. 
\end{example}

The temporal viewpoint lends itself well to the modelling of feedback processes in the life sciences, for which we give two particular examples. 
The first example is the modelling of Boolean networks, a key modelling tool for interactions in systems biology \cite{Kauffman69}.
Its close relationship with logic programming  was first discovered by Inoue \cite{Inoue11}, who showed that point attractors (fixed points) of  Boolean networks correspond to supported models of associated logic programs, and that this correspondence holds whether the influences encoded in the Boolean network are evaluated synchronously or asynchronously.

 The second example are symptom networks studied in psychopathology. 
 In this line of work, championed by Borsboom \cite{Borsboom}, the symptoms associated with a diagnosis are arranged in a network of mutual influences, where feedback loops explain the persistence of symptoms over time, even after initial stimuli have passed. 
 While the modelling approaches used by Borsboon are statistical, a Boolean deterministic approximation by logic programs could potentially lead to a  conceptual clarification, and may be particularly valuable because of the strong individual differentiation between patients and the associated difficulty in obtaining sufficient  quantitative data on a specific patient network. 
Following work on major depressive disorders by Cramer \cite{Cramer13},  Borsboom \cite{Borsboom} makes a distinction between a ground state of a patient and additional possible states that may show symptoms of illness that are not displayed otherwise, arguing that ``additional'' stable states beyond the  healthy ground state are indicative of states of mental disorder. 
Haslbeck et al.\ \cite{HaslbeckRRWB22} evaluate existing modelling paradigms in the field and conclude that none of the graphical model approaches currently in use deals with all of potentially asymmetric relationships, multiple stable states and causal feedback loops. 
This too suggests that logic programming semantics could have a valuable role to play if they were formally aligned with temporal causal processes.  
 
\paragraph{Our contribution}
In this contribution, we take the first technical step towards such applications by clarifying the roles of supported and stable models for modelling causal processes that unfold in time, whether synchronous or asynchronous. 
We show that stable models of positive logic programs correspond to the eventual states of processes that progress from an initial neutral state without interruption, while supported models correspond to the possible eventual states of processes that are eventually compatible with the logic program but have possibly experienced perturbation in the past. 

\paragraph{Related Work}
To the best of our knowledge, the first researchers explicitly to relate causal processes and the semantics of logic programs were Vennekens et al.\ \cite{cplogic}, who define a causal process semantics for probabilistic logic programs. 
In their scheme, processes are inherently finite and every probabilistic rule is only called once. 
This is necessary to ensure that the probabilistic semantics are uniquely defined and align with the distribution semantics for probabilistic logic programs.
However, we are more interested in describing possible eventual states of processes, so in particular our processes are in principle unbounded in time and rules can be called repeatedly to update predicate definitions until possibly a fixed eventual state is reached.
Furthermore, we also want to consider the case where predicate updates occur synchronously rather than successively. 
Hence, we consider a rather different  notion of process compatible with a logic program than that considered by Vennekens et al.\ \cite{cplogic}. 

\section{Preliminaries}

We briefly rehearse notation and fundamental concepts of quantifier-free predicate logic. 
A relational signature is a set of \emph{predicate symbols} $R/n$ of given \emph{arities} $n \in \mathbb{N}_0$, as well as a set of  \emph{constants } $c$. 
$\sigma$-terms are either constants from $\sigma$ or one of a countably infinite set of \emph{variables} $x_i\mid \in \mathbb{N}$. 
They are used to form \emph{atoms} of the form $R(t_1\dots, t_n)$, where $R/n$ is a predicate symbol from $\sigma$ and $t_1, \dots, t_n$ are $\sigma$-terms. 
\emph{Formulas} are constructed from atoms by applying the propositional binary connectives $\land, \lor, \rightarrow$ and $\leftrightarrow$ as well as the unary negation operator $\neg$. 
Additionally, $\top$ and $\bot$ are also formulas. 
Atoms, or formulas of the form $\neg \varphi$, where $\varphi$ is an atom, are called \emph{literals}. 
We call a formula \emph{ground} if no variables occur in it. 
To give meaning to relational formulas, we consider the \emph{Herbrand domain} $D_{\sigma}$ of $\sigma$, which is simply its set of constants. 
Then a \emph{grounding} $\iota$  is a map that allocates to every variable an element of $D_{\sigma}$. 
For every formula $\varphi(\vec{x})$ with free variables from $\vec{x}$, the grounding  $\varphi^\iota$ of $\varphi$ by $\iota$ is obtained from $\varphi(\vec{x})$ by replacing every $x_i$ with $\iota(x_i)$. 
A $\sigma$-structure is any set of ground atoms. 
A ground atom $\varphi$ \emph{holds} in a $\sigma$-structure $\mathfrak{M}$, written $\mathfrak{M}\models \sigma$ , if it is included in $\mathfrak{M}$. 
Other ground formulas are evaluated in $\mathfrak{M}$ according to their ordinary truth tables as propositional connectives (see e.g. \cite{sep-logic-propositional}).

\subsection{Logic programming}
We introduce logic programs and the formal semantics necessary to follow our main results and proofs. For a much more detailed discussion of the semantics of logic programming, a field with a long and distinguished history, see  \cite{DeclarativeLogicProgramming}
\begin{definition}
A \emph{logic program} for a relational signature $\sigma$ is a set of \emph{clauses} $H\leftarrow \mathcal{B}$, where $H$ is a $\sigma$-atom and $\mathcal{B}$ a conjunction of $\sigma$-literals. 
Note that includes the empty conjunction, which evaluates as $\top$; clauses with the empty body are written as \emph{facts}, simply as $H$. 
It is called \emph{positive} if every literal ocurring in any clause body is an atom.

With every logic program $\Pi$, we associate the \emph{dependency graph} $G_\Pi$  whose nodes  are the predicates in $\sigma$ and where there is an edge from $P$ to $R$ if there is a clause $H\leftarrow \mathcal{B}$ in $\Pi$ such that $P$ occurs in the body and $R$ occurs in the head (as its predicate symbol).
We call $\Pi$ \emph{acyclic} if $G_\Pi$ is acyclic.  
\end{definition}

Logic programming clauses can be evaluated on a structure by finding all groundings of clause bodies that evaluate to true and then setting their grounded heads also to be true. 
\begin{definition}
For a relational signature $\sigma$ and a logic program $\Pi$ for $\sigma$, the \emph{one-step-operator} $T_\Pi:D^\sigma \rightarrow D^\sigma$ is defined by 
\[T_\Pi (\mathfrak{M}) = \left\{H^\iota \mid \iota\textrm{ is a grounding to }D^{\sigma},  (H \leftarrow \mathcal{B}) \in \Pi, \mathfrak{M}\models \mathcal{B}^\iota \right\}.\]
\end{definition}

Structures that are invariant under applying the clauses of a program are known as \emph{supported models} and were among the first approaches to giving formal meaning to a logic program. 
\begin{definition}
A $\sigma$-structure $\mathfrak{M}$ is a \emph{supported model} of $\Pi$ if $T_\Pi(\mathfrak{M}) = \mathfrak{M}$.
\end{definition}

Supported models can be characterised logically using Clark's completion semantics \cite{ClarkCompletion}, which associates with every logic program, for every relation symbol $R \in \sigma$,  its \emph{$R$-completion}. 
\begin{definition}
For any predicate $R\in \sigma$ and logic program $\Pi$ for $\sigma$,  the \emph{$R$-completion of $\Pi$} is defined as the set of formulas of the form 
$ A \leftrightarrow \bigvee \mathcal{B}^{\iota}_i$, where $A$ varies over all ground atoms with predicate symbol $R$ and for each $A$, $\mathcal{B}^{\iota}_i$ varies over the clause bodies $\mathcal{B}_i$ and groundings $\iota$ such that there is a clause $H \leftarrow \mathcal{B}_i$ in $\Pi$ where $H^\iota = A$.
Note that if there are no such clauses, the empty disjunction evaluates as $\bot$. 
\end{definition}
Then, a model is a supported model if and only if it satisfies the \emph{$R$-completion} of every predicate $R \in \sigma$. 
Furthermore, Clark showed that if $\Pi$ is acyclic, it has exactly one supported model \cite{ClarkCompletion}. 
 
 \begin{example}\label{ex:burns_supp}
 The $\mathrm{burns}/1$-completion of the burning house program of Example \ref{ex:burns} has the formulas $\mathrm{burns}(h_1) \leftrightarrow \mathrm{burns}(h_2)$ and $\mathrm{burns}(h_2) \leftrightarrow \mathrm{burns}(h_1)$. 
 This leads to the supported models $\emptyset$, where neither house burns, and $\{\mathrm{burns}(h_1), \mathrm{burns}(h_2)\}$, where both houses burn. 
 \end{example}
 
For acyclic programs, there is widespread agreement that the unique supported model is the appropriate semantics for the logic program.
However, for cyclic programs, supported models are generally not unique (see Example \ref{ex:burns_supp}).
In the case of  positive programs, there is still however a unique \emph{least} supported model, where only those propositions are true which are enforced by the program. 
\begin{definition}
Let $\Pi$ be a positive logic program. Then a supported model $\mathfrak{M}$ is called \emph{minimal} or \emph{stable} if for all supported models $\mathfrak{M}'$ of $\Pi$, $\mathfrak{M} \subseteq \mathfrak{M}$. In other words, whenever a ground atom $R(\vec{c})$ is valid in $\mathfrak{M}$, it is also valid in $\mathfrak{M}'$. 
\end{definition}

 \begin{example}\label{ex:burns_stable}
 The unique stable model of the burning house program of Example \ref{ex:burns} is $\emptyset$, where neither house burns. 
 \end{example}
 
For more general programs, stable models as defined by Gelfond and Lifshitz \cite{StableModelSemantics} no longer coincide with minimal supported models and are instead defined as minimal models of relative reducts. 
However, for brevity and since our formal results for stable models are restricted to positive programs, we refrain from presenting the general definition here.

\subsection{Processes}
We now formalise our general notion of a process as a logic program describing a sequence of states ($\sigma$-structures) in discrete time, where the state at any time point depends only on the state in the immediately preceding timepoint. For simplicity, from now on, we restrict our signature $\sigma$ to be finite. 
\begin{definition}\label{def:process}
  Let $\sigma$ be a finite relational vocabulary. Then a \emph{process} for $\sigma$ is a logic program in the vocabulary $\bar{\sigma}:=\{R_t \mid R \in \sigma, t \in \mathbb{N}_0 \}$, where the arities of $R_t$ and $R$ coincide, and with the additional condition that in the body of any clause with head predicate $R_t$, only predicates with index $t-1$ occur.
  In particular, clauses with head predicates $R_0$ must be facts.  

  We call $R_t$ the \emph{$t$-point} of $R$, and we call the set of all clauses of $\mathcal{P}$ whose head predicates are $t$-points the \emph{$t$-step} of $\mathcal{T}$.
  In a wider sense, we refer to the $t$-point of any other expression to mean the expression resulting from replacing all relation symbols from $\sigma$ with their $t$-points. 
\end{definition}

Note that since every head predicate is of a later time-step than every body predicate, 
every process is itself an acyclic logic program (in an infinite vocabulary) and therefore has a unique supported model.

\begin{definition}
  Let $\sigma$ and $\bar{\sigma}$ be as in Definition \ref{def:process} and let $\mathfrak{M}$ be a $\bar{\sigma}$-structure.
  Then for any $t \geq 0$, its \emph{$t$-state} $\mathfrak{M}_t$ is the $\sigma$-model whose true ground atoms are precisely those whose $t$-points are true in $\mathfrak{M}$.
  If there is a $t_0 \geq 0$ such for all $t > t_0$, $\mathfrak{M}_t = \mathfrak{M}_{t_0}$, then $\mathfrak{M}_{t_0}$ is called its \emph{eventual state}.
  If the unique supported model of a process has an eventual state, this is also called the \emph{eventual state} of the process.   
\end{definition}

To relate logic programs for $\sigma$ with processes, we relativise their clauses to a given timepoint. 

\begin{definition}
  Let $C := H \leftarrow \mathcal{B}$ be a clause in $\sigma$.
  Then for any $t > 0$, the \emph{relativisation of $C$ to $t$} is the clause obtained by replacing $H$ with its $t$-point  and $\mathcal{B}$ with its $(t-1)$-point.
\end{definition}

We now define what it means for a process to be adequately described by a logic program.
As we do not wish to place any restrictions on the order in which predicates are updated, which can happen either individually or synchronously, every step is associated with a freely-chosen subvocabulary.
Additionally, since we wish to characterise the eventual behaviour of such processes, it is crucial that all clauses remain relevant indefinitely, or, in other words, that program clauses are invoked infinitely often.
Finally, since propositions should not be true for any reasons outside the process in question, we demand that all propositions start out false. 

\begin{definition}\label{def:compatible}
  Let $\Pi$ be a logic program in   a relational vocabulary $\sigma$.
  Let $(\sigma_t)_{t \in \mathbb{N}_+}$ be a sequence of subsets $\sigma_t \subseteq \sigma$ such that   every $R\in\sigma$ lies in infinitely many $\sigma_t$. 
  Then $\mathcal{T}$ the \emph{$\Pi$-compatible process along $(\sigma_t)_{t>0}$} is defined by the following clauses. Any such process is also called simply \emph{compatible with $\Pi$}.
  \begin{enumerate}
  \item The $0$-step of $\mathcal{T}$ is the empty program.
  \item For all $t > 0$, the $t$-step of $\mathcal{T}$ is defined by the following clauses:
    \begin{itemize}
    \item For all $R\in\sigma_t$, the relativisation of each clause in $\Pi$ with head predicate $R$ to $t$;
    \item For all $R\notin\sigma_t$, the clause $R_t(\vec{x}) \leftarrow R_{t-1}( \vec{x})$. 
    \end{itemize}
  \end{enumerate}
\end{definition}

\begin{example}\label{def:burns_process}
A process compatible with the burning houses program $\Pi$ of Example \ref{ex:burns} would have, for each time $t$, either $\sigma_t = \sigma$ or $\sigma_t = \emptyset$, where the former occurs infinitely often. 
The $t$-step is either the relativisation of $\Pi$ to $t$ or the program with the single clause $\mathrm{burns}_t(x) \leftarrow \mathrm{burns}_{t-1}(x)$. 
The $0$-step is $\emptyset$, and thus the $0$-state is also the empty set. 
One can see by a simple induction that in the supported model of this process, all states must be $\emptyset$, as the right-hand side of each equivalence of the completion evaluates as false. 
\end{example}

We now turn to processes whose eventual behaviour is characterised by a logic program, but which may have arbitrary behaviour before then.
Note that every step of the compatible program only makes direct reference to the state immediately preceding. 
Hence, in terms of their eventual behaviour, only the final state of the previous process before entering compatibility with the logic program matters. 
We can therefore simplify further analysis and assume that the final state of that previous process is the 0-state, and that therefore the eventually compatible process differs from a fully compatible one only in that 0-state. 

\begin{definition}  
For any $(\sigma_t)_{t \in \mathbb{N}_+}$ a sequence of subsets $\sigma_t \subseteq \sigma$ such that  every $R\in\sigma$ lies in infinitely many $\sigma_t$ and every $\sigma$-structure $\mathbf{M}$, the  \emph{eventually $\Pi$-compatible process along  $(\sigma_t)_{t  > 0}$  with initial state $\mathfrak{M}$} is the process whose 0-state contains facts corresponding precisely to the ground atoms true in $\mathfrak{M}$ and whose $t$-states for $t>0$ satisfy Point 2  of Definition \ref{def:compatible}.  Such a process is also simply called \emph{eventually compatible with $\Pi$}.
\end{definition}

\begin{example}
Consider a process eventually compatible with the burning houses program $\Pi$ of Example \ref{ex:burns} differs from the description in Example \ref{def:burns_process}. 
If the $0$-state is instead the model where both houses are burning, then this is conserved under the $t$-step regardless of $\sigma_t$. 
Hence, it is also the eventual state.  
On the other hand, if the $0$-state is either $\{\mathrm{burns}(h_1)\}$ or $\{\mathrm{burns}(h_2)\}$, the process will oscillate to the other one of those two states at every $t$ where $\sigma_t = \sigma$. 
As this happens infinitely often, the process does not have an eventual state. 
If $\Pi$ had an additional clause that ensured the persistency of at least one of the ground atoms, such as  $\mathrm{burns}(X) \leftarrow \mathrm{burns}(X)$, $\mathrm{burns}(h_1) \leftarrow \mathrm{burns}(h_1)$ or $\mathrm{burns}(h_2) \leftarrow \mathrm{burns}(h_2)$, then the process would eventually reach the state where both houses are burning and remain there indefinitely. 
\end{example}

\section{Main results}

\subsection{General logic programs}
We turn to the main results characterising supported models and stable models for positive programs. 
As the result for supported models holds in general logic programs, and as logic programs with negation are key to the representation of Boolean networks \cite{Inoue11}, we give this result in full generality. 
This is a slight generalisation of the corresponding results of Inoue \cite{Inoue11}, phrased there in the language of Boolean networks. 
\begin{theorem}\label{thm:supported}
  Let $\Pi$ be a logic program and $\mathcal{T}$ be eventually compatible with  $\Pi$.
  If it exists, the eventual state of $\mathcal{T}$ is then a supported model of $\Pi$, and every supported model of $\Pi$ is the eventual state of a process eventually compatible with $\Pi$.
\end{theorem}

\begin{proof}
  Let $\Pi$ be a logic program and $\mathcal{T}$ be eventually compatible with  $\Pi$, and let $\mathfrak{M}_{t_0}$ be the eventual state of $\mathcal{T}$.
  We want to show that $\mathfrak{M}_{t_0}$ is a supported model of $\Pi$.
  So let $R\in \sigma$ be a predicate symbol.
  Then there is a $t \geq t_0$ such that $R\in \sigma_{t+1}$.
  Since $\mathfrak{M}$ is a supported model of $\mathcal{T}$, we know that $\mathfrak{M}$ satisfies the formulas $(R_{t+1}(\vec{s}) \leftrightarrow \mathcal{B}_t)$ of the $R_{t+1}$-completion of $\mathcal{T}$.
  Furthermore, since $\mathfrak{M}_{t_0}$ is the eventual state, we know that the $t$-points and the $t+1$-points of all relations in $\sigma$ coincide with their $t_0$-points.
  Hence, $\mathfrak{M}_{t_0}$ satisfies the formulas $(R(\vec{s}) \leftrightarrow \mathcal{B})$ of the $R$-completion of $\Pi$.
  Since $R$ was arbitrarily chosen, this suffices.
  Conversely, it is clear that every supported model $\mathfrak{M}$ is also the eventual state of an eventually compatible process, since we simply need to set the $0$-step of $\mathcal{T}$ include facts for every positive atom true in $\mathfrak{M}$ and then set every other $\sigma_t$ to $\sigma$.
  Then every $t$-step of $\mathcal{T}$ corresponds to one application of the one-step operator of $\Pi$, whose fixed points are precisely the supported models.
\end{proof}

We now turn to the characterisation of compatible processes as stable models. 
As this characterisation does not generalise naively to more general classes of programs (see Example \ref{ex:counterexample} below), we now restrict our attention to positive logic programs.

\begin{theorem}\label{thm:stable}
  Let $\Pi$ be a positive logic program and $\mathcal{T}$ be compatible with $\Pi$.
  Then the eventual state of $\mathcal{T}$ is the unique stable model of $\Pi$.
\end{theorem}

\begin{proof}
  The stable model of $\Pi$ is the least fixed point of the one-step operator of $\Pi$, reached after finitely many applications (say, $t_0$ applications).
  Let $\mathcal{T}'$ be the compatible process for which $\sigma_t(\mathcal{T}') = \sigma$ for every $t>0$, and let $\mathfrak{M}'$ be its model.
  As every step of $\mathcal{T}'$ corresponds to one application of the one-step operator, $\mathfrak{M}'_t$ is the state after $t$ applications of the one-step operator, and the eventual state of $\mathcal{T}$ is exactly the stable model of $\Pi$.  
  
  Furhermore, for any $t \geq 0$, let $\mathrm{index}_{\mathcal{T}}(t)$ be defined recursively as follows:
  \begin{itemize}
  \item For all $t \geq 0$,  $\mathrm{index}_{\mathcal{T}}(t) \geq 0$.
  \item For all $i > 0$ and  $t > 0$, $\mathrm{index}_{\mathcal{T}}(t) \geq i$ if there is an $s < t$ with  $\mathrm{index}_{\mathcal{T}}(0) \geq i-1$ such that for every $R\in \sigma$ there is a $s < t_R \leq t $ such that $R\in \sigma_{t_R}{\mathcal{T}}$.
  \item For all  $t \geq 0$, $\mathrm{index}_{\mathcal{T}}(t) = i$ if $\mathrm{index}_{\mathcal{T}}(t) \geq i$ and not $\mathrm{index}_{\mathcal{T}}(t) \geq i+1$.
  \end{itemize}
  The index is always well-defined and at most $t$, and since every $R$ occurs in infinitely many $\sigma_t$, the index is also always unbounded.  
    We show that for every compatible process $\mathcal{T}$ and every $t > 0$, $\mathfrak{M}'_{\mathrm{index}_{\mathcal{T}}(t)} \leq \mathfrak{M}_t \leq \mathfrak{M}'_t$.
  This then implies that for every $t$ with $\mathrm{index}_{\mathcal{T}}(t) \geq t_0$, $\mathfrak{M}_t$ is the stable model of $\Pi$.
  As the index is unbounded, this suffices to prove that the stable model of $\Pi$ is the eventual state of $\mathcal{T}$.

  So it remains to show that for every $t > 0$, $\mathfrak{M}'_{\mathrm{index}_{\mathcal{T}}(t)} \leq \mathfrak{M}_t \leq \mathfrak{M}'_t$.
  We proceed by induction on $t$.
  We first show that this is true for $t = 1$.
  If $\sigma_1(\mathcal{T}) = \sigma$, then $\mathrm{index}_{\mathcal{T}}(1) = 1$ and both inequalities collapse to equality.
  If $\sigma_1(\mathcal{T}) \neq \sigma$, then $\mathrm{index}_{\mathcal{T}}(1) = 0$.
  Hence, the first inequality holds as no ground atoms are true in $\mathfrak{M}'_0$ and the second inequality holds true since for all $R\in \sigma_1$, their interpretations in $\mathfrak{M}_1$ and $\mathfrak{M}'_1$ coincide, while for all $R\notin \sigma_1$, no ground atoms with those predicates hold in $\mathfrak{M}_1$.
  So assume that the statement holds for all $s < t$.
  We proceed to show the statement for $t$.
  Note that since $\Pi$ is positive, $\mathfrak{M}_s \leq \mathfrak{M}_t$ for all $s<t$, and likewise for $\mathfrak{M}'_s$ and $\mathfrak{M}'_t$.

  We first show $\mathfrak{M}'_{\mathrm{index}_{\mathcal{T}}(t)} \leq \mathfrak{M}_t$.
  So let $i := \mathrm{index}_{\mathcal{T}}(t) > 0$.
  We need to show that for every ground atom $R(\vec{c})$, $R(\vec{c})$ is true in $\mathfrak{M}_t$ if it is true in  $\mathfrak{M}'_{\mathrm{index}_{\mathcal{T}}(t)}$.
  Let $R\in\sigma$ and let $s,t_R$ be as in the definition of the index of $t$.
  Then $\mathfrak{M}'_{i-1} \leq \mathfrak{M}_s \leq \mathfrak{M}_{t_R-1}$ and thus since $\Pi$ is positive, every $R(\vec{c})$ true in $\mathfrak{M}'_{i}$ is also true in $\mathfrak{M}_{t_R}$ (as the result of applying the same positive clauses of $\Pi$ to $\mathfrak{M}'_{i-1}$ and $\mathfrak{M}_{t_R-1}$) and hence in $\mathfrak{M}_t$.    

  We then turn to $\mathfrak{M}_t \leq \mathfrak{M}'_t$.
  This follows immediately from the induction hypothesis applied to $t-1$, though, since all clauses in $\Pi$ are positive. 
  We have therefore shown that when $\Pi$ is positive, every process compatible with $\Pi$ has an eventual state equal to the stable model of $\Pi$, as desired.
\end{proof}


\section{Conclusion and outlook}
We showed that when viewing (positive) logic programs as shorthand for classes of processes, stable and supported models have clearly demarkated roles.
Stable models correspond to fixed points of  processes allowed to proceeed undisturbed from an initial state of neutrality (where none of the propositions in the signature are true).
Supported models, on the other hand, correspond to arbitrary fixed points of processes, which may be reached by perturbations of the initial state. 
In this way, they also capture latent cycles, which may not be activated in the ordinary course of affairs but which may be reached when extraordinary perturbations lead to a change of state at even a single time point. 

Our results beg the question whether they could be extended beyond positive programs and processes reaching an unchanging eventual state. 
Regarding the first point, the situation is more complex already for stratified logic programs, and the direct analogue of Theorem \ref{thm:stable} can fail even when stable and minimal models still coincide:

\begin{example}\label{ex:counterexample}
Let $\Pi$ be the logic program with clauses $a.$, $b \leftarrow \neg a.$ and $b \leftarrow b$. 
This (stratified) program has a unique stable (and minimal) model, where $a$ is true and $b$ is false, and one additional supported model where both $a$ and $b$ are true. 
Unlike in the situation of Theorem \ref{thm:stable}, both supported models occur as fixed points of compatible processes. 
Consider the process where for all $t > 0$, $\sigma_t := \sigma$.
Then, since $a$ is false in the $0$-state, $a$ and $b$ are true in the $1$-state. 
Both propositions remain true in all subsequent states, and thus its eventual state is the model where both propositions are true.   
Now consider the process where $\sigma_1 := \{a\}$ and for all $t > 1$, $\sigma_t := \sigma$.
In this case, only $a$ is true in the 1-state, and since the 1-state is already a supported (and indeed the stable) model, it is also its eventual state. 
\end{example}

This could be remedied by restricting the definition of a compatible process for stratified programs to respect the order of stratification, as in the probabilistic process semantics of Vennekens et al. \cite{cplogic}. 
However, for the purposes of causal process modelling outlined in the introduction, it seems more important to describe a broad, natural and general class of processes rather than try and ensure a uniquely defined eventual state semantics. 

Going  beyond eventual states to processes in which perhaps only the truth values of some propositions are eventually fixed, while others oscillate, seems potentially a very fruitful endeavour. 
This situation seems to correspond naturally to a three-valued semantics, so it would be particularly interesting to see whether Theorem \ref{thm:supported} generalises to the three-valued supported models of Fitting and Kunen \cite{Fitting85,Kunen87}.
Such a characterisation could be a powerful knowledge representation tool, as it would alllow for formalising the possible eventual behaviour of eventually compatible processes in full generality. 

Another promising extension that would align our work more closely with the statistical approach to network models prevalent in much of the literature would be to extend the models with probabilistic components. 
In principle, the distribution semantics \cite{distribution_semantics,distribution_semantics1} provides a flexible way of attaching probabilities to generic logic programming frameowrks. 
However, simply attaching probabilistic facts to the original program and evaluating it under the distribution semantics would lead back to the set-up of Vennekens et al.\ \cite{cplogic}, which does not match well with describing patterns of repeated probabilistic application of the same rules. 

\subsection*{Acknowledgements}
I would like to thank Julia v.\ Thienen, Kilian R\"uckschlo{\ss} and Kailin Weitk\"amper for fruitful discussions. 

\bibliographystyle{eptcs}
\bibliography{generic}

\begin{thebibliography}{10}
\providecommand{\bibitemdeclare}[2]{}
\providecommand{\surnamestart}{}
\providecommand{\surnameend}{}
\providecommand{\urlprefix}{Available at }
\providecommand{\url}[1]{\texttt{#1}}
\providecommand{\href}[2]{\texttt{#2}}
\providecommand{\urlalt}[2]{\href{#1}{#2}}
\providecommand{\doi}[1]{doi:\urlalt{https://doi.org/#1}{#1}}
\providecommand{\eprint}[1]{arXiv:\urlalt{https://arxiv.org/abs/#1}{#1}}
\providecommand{\bibinfo}[2]{#2}

\bibitemdeclare{book}{bochman}
\bibitem{bochman}
\bibinfo{author}{Alexander \surnamestart Bochman\surnameend}
  (\bibinfo{year}{2021}): \emph{\bibinfo{title}{{A Logical Theory of
  Causality}}}.
\newblock \bibinfo{publisher}{The MIT Press}.
\newblock \urlprefix\url{https://doi.org/10.7551/mitpress/12387.001.0001}.

\bibitemdeclare{article}{Borsboom}
\bibitem{Borsboom}
\bibinfo{author}{Denny \surnamestart Borsboom\surnameend}
  (\bibinfo{year}{2017}): \emph{\bibinfo{title}{A network theory of mental
  disorders}}.
\newblock {\slshape \bibinfo{journal}{World Psychiatry}}
  \bibinfo{volume}{16}(\bibinfo{number}{1}), pp. \bibinfo{pages}{5--13},
  \doi{10.1002/wps.20375}.

\bibitemdeclare{inproceedings}{ClarkCompletion}
\bibitem{ClarkCompletion}
\bibinfo{author}{Keith~L. \surnamestart Clark\surnameend}
  (\bibinfo{year}{1978}): \emph{\bibinfo{title}{Negation as Failure}}.
\newblock In: {\slshape \bibinfo{booktitle}{Logic and Data Bases}},
  \bibinfo{publisher}{Springer US}, \bibinfo{address}{Boston, MA}, pp.
  \bibinfo{pages}{293--322}, \doi{10.1007/978-1-4684-3384-5_11}.

\bibitemdeclare{misc}{Cramer13}
\bibitem{Cramer13}
\bibinfo{author}{Angelique \surnamestart Cramer\surnameend}
  (\bibinfo{year}{2013}): \emph{\bibinfo{title}{The glue of (ab)normal mental
  life: Networks of interacting thoughts, feelings and behaviors}}.
\newblock \urlprefix\url{https://hdl.handle.net/11245/1.397502}.
\newblock \bibinfo{note}{Doctoral thesis, Amsterdam}.

\bibitemdeclare{article}{EelinkRW26}
\bibitem{EelinkRW26}
\bibinfo{author}{Guus \surnamestart Eelink\surnameend}, \bibinfo{author}{Kilian
  \surnamestart Rückschloß\surnameend} \& \bibinfo{author}{Felix
  \surnamestart Weitkämper\surnameend} (\bibinfo{year}{2026}):
  \emph{\bibinfo{title}{How artificial intelligence leads to knowledge why: An
  inquiry inspired by Aristotle’s Posterior Analytics}}.
\newblock {\slshape \bibinfo{journal}{Int. J. Approx. Reason.}}
  \bibinfo{volume}{190}(\bibinfo{number}{109603}),
  \doi{10.1016/j.ijar.2025.109603}.

\bibitemdeclare{article}{Fitting85}
\bibitem{Fitting85}
\bibinfo{author}{Melvin \surnamestart Fitting\surnameend}
  (\bibinfo{year}{1985}): \emph{\bibinfo{title}{A Kripke-Kleene semantics for
  logic programs}}.
\newblock {\slshape \bibinfo{journal}{J. Log. Program.}}
  \bibinfo{volume}{2}(\bibinfo{number}{4}), pp. \bibinfo{pages}{295--312},
  \doi{10.1016/S0743-1066(85)80005-4}.

\bibitemdeclare{incollection}{sep-logic-propositional}
\bibitem{sep-logic-propositional}
\bibinfo{author}{Curtis \surnamestart Franks\surnameend}
  (\bibinfo{year}{2024}): \emph{\bibinfo{title}{{Propositional Logic}}}.
\newblock In: {\slshape \bibinfo{booktitle}{The {Stanford} Encyclopedia of
  Philosophy}}, \bibinfo{edition}{{W}inter 2024} edition,
  \bibinfo{publisher}{Metaphysics Research Lab, Stanford University}.
\newblock
  \urlprefix\url{https://plato.stanford.edu/archives/win2024/entries/logic-propositional/}.

\bibitemdeclare{inproceedings}{StableModelSemantics}
\bibitem{StableModelSemantics}
\bibinfo{author}{Michael \surnamestart Gelfond\surnameend} \&
  \bibinfo{author}{Vladimir \surnamestart Lifschitz\surnameend}
  (\bibinfo{year}{1988}): \emph{\bibinfo{title}{The Stable Model Semantics for
  Logic Programming}}.
\newblock In: {\slshape \bibinfo{booktitle}{Proceedings of International Logic
  Programming Conference and Symposium}}, \bibinfo{publisher}{MIT Press}, pp.
  \bibinfo{pages}{1070--1080}.
\newblock \urlprefix\url{http://www.cs.utexas.edu/users/ai-lab?gel88}.

\bibitemdeclare{article}{HaslbeckRRWB22}
\bibitem{HaslbeckRRWB22}
\bibinfo{author}{Jonas \surnamestart Haslbeck\surnameend},
  \bibinfo{author}{Oisin \surnamestart Ryan\surnameend},
  \bibinfo{author}{Donald~J \surnamestart Robinaugh\surnameend},
  \bibinfo{author}{Lourens~J \surnamestart Waldorp\surnameend} \&
  \bibinfo{author}{Denny \surnamestart Borsboom\surnameend}
  (\bibinfo{year}{2022}): \emph{\bibinfo{title}{Modeling psychopathology: From
  data models to formal theories.}}
\newblock {\slshape \bibinfo{journal}{Psychol. Methods}}
  \bibinfo{volume}{27}(\bibinfo{number}{6}), p. \bibinfo{pages}{930},
  \doi{10.1037/met0000303}.

\bibitemdeclare{inproceedings}{Inoue11}
\bibitem{Inoue11}
\bibinfo{author}{Katsumi \surnamestart Inoue\surnameend}
  (\bibinfo{year}{2011}): \emph{\bibinfo{title}{Logic programming for Boolean
  networks}}.
\newblock In: {\slshape \bibinfo{booktitle}{Proceedings of the Twenty-Second
  International Joint Conference on Artificial Intelligence - Volume Volume
  Two}}, \bibinfo{series}{IJCAI'11}, \bibinfo{publisher}{AAAI Press}, p.
  \bibinfo{pages}{924–930}, \doi{10.5591/978-1-57735-516-8/IJCAI11-160}.

\bibitemdeclare{article}{Kauffman69}
\bibitem{Kauffman69}
\bibinfo{author}{Stuart \surnamestart Kauffman\surnameend}
  (\bibinfo{year}{1969}): \emph{\bibinfo{title}{Homeostasis and differentiation
  in random genetic control networks}}.
\newblock {\slshape \bibinfo{journal}{Nature}}
  \bibinfo{volume}{224}(\bibinfo{number}{5215}), pp. \bibinfo{pages}{177--178},
  \doi{10.1038/224177a0}.

\bibitemdeclare{book}{DeclarativeLogicProgramming}
\bibitem{DeclarativeLogicProgramming}
\bibinfo{editor}{Michael \surnamestart Kifer\surnameend} \&
  \bibinfo{editor}{Yanhong~Annie \surnamestart Liu\surnameend}, editors
  (\bibinfo{year}{2018}): \emph{\bibinfo{title}{Declarative logic programming:
  theory, systems, and applications}}.
\newblock \bibinfo{publisher}{Association for Computing Machinery and Morgan \&
  Claypool}.

\bibitemdeclare{article}{Kunen87}
\bibitem{Kunen87}
\bibinfo{author}{Kenneth \surnamestart Kunen\surnameend}
  (\bibinfo{year}{1987}): \emph{\bibinfo{title}{Negation in logic
  programming}}.
\newblock {\slshape \bibinfo{journal}{J. Log. Program.}}
  \bibinfo{volume}{4}(\bibinfo{number}{4}), pp. \bibinfo{pages}{289--308},
  \doi{10.1016/0743-1066(87)90007-0}.

\bibitemdeclare{book}{Causality}
\bibitem{Causality}
\bibinfo{author}{J.~\surnamestart Pearl\surnameend} (\bibinfo{year}{2000}):
  \emph{\bibinfo{title}{{C}ausality}}, \bibinfo{edition}{2nd} edition.
\newblock \bibinfo{publisher}{Cambridge University Press},
  \doi{10.1017/CBO9780511803161}.

\bibitemdeclare{article}{distribution_semantics1}
\bibitem{distribution_semantics1}
\bibinfo{author}{David \surnamestart Poole\surnameend} (\bibinfo{year}{1993}):
  \emph{\bibinfo{title}{Probabilistic Horn abduction and {B}ayesian networks}}.
\newblock {\slshape \bibinfo{journal}{Artificial Intelligence}}
  \bibinfo{volume}{64}(\bibinfo{number}{1}), pp. \bibinfo{pages}{81--129},
  \doi{10.1016/0004-3702(93)90061-F}.

\bibitemdeclare{inproceedings}{RueckschlossWeitkaemper:2025}
\bibitem{RueckschlossWeitkaemper:2025}
\bibinfo{author}{Kilian \surnamestart Rückschloß\surnameend} \&
  \bibinfo{author}{Felix \surnamestart Weitkämper\surnameend}
  (\bibinfo{year}{2025}): \emph{\bibinfo{title}{How Rules Represent Causal
  Knowledge: Causal Modeling with Abductive Logic Programs}}.
\newblock In: {\slshape \bibinfo{booktitle}{RuleML+RR-Companion 2025}},
  {\slshape \bibinfo{series}{CEUR Workshop Proceedings}}
  \bibinfo{volume}{4083}, \bibinfo{publisher}{CEUR-WS.org},
  \bibinfo{address}{Istanbul, Türkiye}.
\newblock \urlprefix\url{https://ceur-ws.org/Vol-4083/paper62.pdf}.

\bibitemdeclare{inproceedings}{distribution_semantics}
\bibitem{distribution_semantics}
\bibinfo{author}{Taisuke \surnamestart Sato\surnameend} (\bibinfo{year}{1995}):
  \emph{\bibinfo{title}{A Statistical Learning Method for Logic Programs with
  Distribution Semantics}}.
\newblock In: {\slshape \bibinfo{booktitle}{Proc. ICLP `95}},
  \bibinfo{publisher}{MIT Press}, pp. \bibinfo{pages}{715--729}.

\bibitemdeclare{article}{cplogic}
\bibitem{cplogic}
\bibinfo{author}{J.~\surnamestart Vennekens\surnameend},
  \bibinfo{author}{M.~\surnamestart Denecker\surnameend} \&
  \bibinfo{author}{M.~\surnamestart Bruynooghe\surnameend}
  (\bibinfo{year}{2009}): \emph{\bibinfo{title}{{C}{P}-logic: A Language of
  Causal Probabilistic Events and Its Relation to Logic Programming}}.
\newblock {\slshape \bibinfo{journal}{Theory Pract. Log. Program.}}
  \bibinfo{volume}{9}(\bibinfo{number}{3}), p. \bibinfo{pages}{245–308},
  \doi{10.1017/S1471068409003767}.

\end{thebibliography}
\end{document}